\documentclass{article}

\usepackage[dblblindworkshop, preprint]{neurips_2026}

\workshoptitle{Geometric Distributional Deep
Learning.}

\usepackage[utf8]{inputenc}
\usepackage[T1]{fontenc}
\usepackage{amsmath}
\usepackage{amssymb}
\usepackage{amsthm}
\usepackage{mathtools}
\usepackage{booktabs}
\usepackage{fontawesome5}
\usepackage{microtype}
\usepackage{xcolor}
\usepackage{url}
\usepackage{hyperref}
\usepackage{algorithm}
\usepackage{algpseudocode} 
\usepackage{multirow} 
\usepackage{thmtools}
\usepackage{thm-restate}
\usepackage{placeins}
\usepackage{graphicx}
\usepackage{subcaption}
\usepackage{comment}
\usepackage{centernot}
\usepackage{cancel}

\newcommand{\KL}{\mathrm{KL}}
\newcommand{\GW}{\mathrm{GW}}
\newcommand{\sg}{\operatorname{sg}}
\DeclareMathOperator*{\argmin}{arg\,min}
\DeclareMathOperator*{\argmax}{arg\,max}

\newtheorem{proposition}{Proposition}[section]
\theoremstyle{remark}

\title{Graph Matching Relaxations and Amortization for Supervised Graph Prediction}

\author{%
  \begin{tabular}{ccc}
    \makebox[0.30\textwidth][c]{%
      Federico Méndez\textsuperscript{$\star$}\textsuperscript{\textnormal{1,3}}}
    &
    \makebox[0.30\textwidth][c]{%
      Paul Krzakala\textsuperscript{$\star$}\textsuperscript{\textnormal{1,2,3}}}
    &
    \makebox[0.30\textwidth][c]{%
      Gabriel Melo\textsuperscript{$\star$}\textsuperscript{\textnormal{1,3}}}
    \\[0.6em]
    \makebox[0.30\textwidth][c]{%
      Charlotte Laclau\textsuperscript{\textnormal{1,3}}}
    &
    \makebox[0.30\textwidth][c]{%
      Rémi Flamary\textsuperscript{\textnormal{2,3}}}
    &
    \makebox[0.30\textwidth][c]{%
      Florence d'Alché-Buc\textsuperscript{\textnormal{1,3}}}
  \end{tabular}
  \\[1.5em]
  \textsuperscript{1}LTCI, Télécom Paris
  \hspace{0.3cm}
  \textsuperscript{2}CMAP, École Polytechnique
  \hspace{0.3cm}
  \textsuperscript{3}Institut Polytechnique de Paris
  \\[0.2em]
  \textsuperscript{$\star$}Equal Contribution
}

\begin{document}
\maketitle
\begin{abstract}
End-to-end Supervised Graph Prediction (SGP) requires a
permutation-invariant loss to compare predicted and target graphs with
arbitrary node orderings. Such losses typically involve a costly graph-matching
problem. We first study three Optimal Transport relaxations of this problem and show,
theoretically and empirically, that the Gromov-Wasserstein (GW) objective is
the most suitable for SGP. Then, to avoid solving the resulting inner optimization for
every training example, we propose to amortize the graph matching (node alignment) problem.
For each training sample, the loss function leverages a transport plan provided by a parametric 
matcher based on the differentiable Sinkhorn algorithm applied on empirical node distributions. 
The graph prediction module and the matcher are jointly learned. 
We showcase the efficiency of this approach on toy and real world SGP problems of increasing complexity 
including a novel Mass-spectra to Scaffold task that we introduce.\end{abstract}

\section{Introduction}
Graphs provide a powerful and widely used tool to represent structured objects in various domains such as chemistry (molecules) or digital humanities (social networks) \citep{hu2020open, zhu2022survey, nguyen2019recent, tang2008arnetminer}. While most graph machine learning focus on predicting graph properties with a graph as the input variable \cite{wu2018moleculenet, dwivedi2022long}, we consider instead the task of predicting an entire graph from an input variable which is not necessarily a graph. We refer to this setting as Supervised Graph Prediction (SGP).
A flagship example of SGP task is \textit{de novo} molecule identification, where the goal is to reconstruct a molecular graph of an unknown compound from a spectra acquired by tandem Mass Spectrometry  \citep{alseekh2021mass, bushuiev2024massspecgym,brogat2022learning, melo2026conformal, nguyen2019recent}.

This \emph{supervised graph prediction} task poses a core
difficulty as graphs have no inherent node ordering which means that the predicted and target graph cannot be compared by a classical entry-by-entry data fitting term. Computationally expensive permutation-invariant metrics are thus required instead.

Computing such loss, typically involves an alignment steps where an optimal correspondence is found between the nodes of the two graphs which results in a combinatorial graph matching problem \cite{burkard1984quadratic,koopmans1957assignment}. In practice, the problem is often relaxed to be tackled by continuous solvers \cite{leordeanu2005spectral,zaslavskiy2008path}. In particular, the bistochastic relaxation is a natural choice that connects the graph matching and Optimal Transport (OT) literature through the celebrated Gromov-Wasserstein distance \citep{peyre2016gromov, memoli2011gromov} and its variants \cite{titouan2019optimal,yang2024exploiting}. This approach was successfully applied to the training of end-to-end models SGP models such as  Any2Graph \citep{krzakala2024any2graph}. 

Despite the growing efficiency of the OT solvers \cite{scetbon2022linear}, matching
every predicted-target pair quickly becomes a bottleneck as graphs and dataset size grow. An appealing
alternative is to amortize the matching step by training a parametric network to predict the alignment
directly \cite{krzakala2025grale,mazelet2026unsupervised}. This approach was originally introduced into a graph-level autoencoder, GRALE \citep{krzakala2025grale} but remains unexplored for supervised prediction, where the
matcher should be optimized for the predictive task rather than reconstruction.

In this paper we close this gap by extending the "learning to match" framework to Supervised Graph Prediction. Our contributions are listed below:
\begin{itemize}
    \item We provide empirical and theoretical evidence that the Gromov-Wasserstein distance outperforms the alternative bistochastic relaxations of combinatorial graph matching for SGP. 
    \item We introduce an amortized approach to SGP where a matcher is jointly trained to predict the optimal matching instead of relying on a solver.
    \item We propose a regularization term that penalizes the matcher for producing non-bistochastic matching which simplifies the matcher design and improves performance.
    \item We demonstrate that the amortized approach consistently outperforms the solver alternative both in compute time and prediction performance.
    \item Finally, the experiments on both synthetic datasets and two challenging molecular task show that the proposed model outperforms existing SGP models.
\end{itemize}
\section{Problem Setup}
\label{sec:problem_setup}

\paragraph{Notation.} We write $\mathbf{1}_n$ for the all-ones vector,
$\langle U,V\rangle=\operatorname{tr}(U^\top V)$ for the Frobenius
inner product, and
$\Pi_n
=
\left\{
T\in[0,1]^{n\times n}:
T\mathbf{1}_n=\mathbf{1}_n,\;
T^\top\mathbf{1}_n=\mathbf{1}_n
\right\}$ for the Birkhoff polytope of bistochastic matrices. Its extreme
points are permutation matrices, whose set we denote $\Sigma_n\subset\Pi_n$. 

\paragraph{Supervised Graph Prediction} We consider the supervised graph prediction problem of learning a mapping $f_\theta:\mathcal{X}\to\mathcal{G}$ from an input space $\mathcal{X}$ to the space of graphs $\mathcal{G}$. Given a training set $\{(x_i, g_i^\star)\}_{i=1}^N$, the parameters $\theta$ are fit by minimizing the empirical risk \begin{equation}
    \mathcal{L}(\theta) = \sum_{i=1}^N \ell\big(f_\theta(x_i),\, g_i^\star\big),
\end{equation}
where $\ell:\mathcal{G}\times\mathcal{G}\to\mathbb{R}_+$ is a permutation-invariant graph loss, i.e., invariant to reordering of the nodes. For the remainder of the paper, we represent a graph by its adjacency matrix and take $\mathcal{G}=[0,1]^{n\times n}$, where $n$ denotes a fixed maximum graph size (smaller graphs being zero-padded); we write $A_i$ for the adjacency matrix of $g_i$. Appendix~\ref{extension_labeled} discuss the extension to labeled graphs of arbitrary size.

\paragraph{Graph matching.}The loss $\ell$ must be invariant to node reordering,
since a graph is unchanged by permuting its node indices \cite{titouan2019optimal}. A natural such loss
compares two adjacency matrices under the best alignment of their nodes. For adjacency matrices
\(A,B\in[0,1]^{n\times n}\), their graph-matching loss is
\begin{equation}
    \ell(A,B)
    = 
    \min_{P\in\Sigma_n}
    J(P;A,B),
    \qquad
    J(P;A,B)
    \coloneqq
    \sum_{i,j,k,l=1}^n
    d(A_{ik},B_{jl})P_{ij}P_{kl},
    \label{eq:qap}
\end{equation}
and \(d:[0,1]\times[0,1]\to\mathbb{R}_+\) is an edge-wise discrepancy.
This is a quadratic assignment problem (QAP) \cite{koopmans1957assignment, burkard1984quadratic}. For a permutation matrix \(P\in \Sigma_n\), the same objective can equivalently be
written as
\begin{equation}
    J_a(P;A,B)
    \coloneqq
    \sum_{i,l=1}^n
    d\bigl([AP]_{il},[PB]_{il}\bigr)
    \;\; \text{and} \;\;
    J_b(P;A,B)
    \coloneqq
    \sum_{i,k=1}^n
    d\bigl(A_{ik},[PBP^\top]_{ik}\bigr).
\end{equation}
\begin{restatable}[Equivalent permutation objectives]{proposition}{equivmatching}
\label{prop:equivalent-matching-objectives}
For every permutation matrix \(P\in\Sigma_n\) and every pair of adjacency
matrices \(A,B\in[0,1]^{n\times n}\),
\[
    J(P;A,B) = J_a(P;A,B) = J_b(P;A,B).
\]
\end{restatable}
The three objectives are therefore equivalent over permutations, even though they yield different relaxations as discussed latter. Unfortunately, the optimization problem in \eqref{eq:qap} is NP-complete in general \cite{hartmanis1982computers} which raises a first question:
\begin{center}
\textit{How to efficiently compute or approximate $\ell(A,B)$ to produce a practical loss for SGP?}
\end{center}
We address this question in the next section. The detailed proofs for this section and the next are available in Appendix~\ref{app:proofs}.

\section{Efficient Computation: From Relaxation to Amortization}
In this section, we present the efficient computation of the graph matching loss, building from Optimal Transport (bistochastic) relaxation, to parallelizable solvers to the proposed approach: amortization.
\label{sec:efficient}
\paragraph{Optimal Transport Relaxation.} A standard approach to approximate the QAP problem is to relax the discrete set \(\Sigma_n\) by its convex hull, \(\Pi_n\) \cite{zaslavskiy2008path}. Interestingly, even though objectives \(J, J_a, J_b\) coincide on \(\Sigma_n\)
(Proposition~\ref{prop:equivalent-matching-objectives}) they actually differ on
\(\Pi_n\) (Proposition~\ref{prop:not-equivalent}), which yields 3 distinct relaxations:
\begin{align}
    \GW(A,B)
    \coloneqq
    \min_{T\in\Pi_n} J(T;A,B).
    \label{eq:gw-relaxation}
\end{align}
\begin{align}
    \ell_a(A,B) \coloneqq \min_{T\in\Pi_n} J_a(T;A,B) \quad \ell_b(A,B) \coloneqq \min_{T\in\Pi_n} J_b(T;A,B)
\end{align}
Where $\GW(A,B)$ is known in the Optimal Transport field as the discrete Gromov-Wasserstein distance
\citep{memoli2011gromov}. This raises the second question of this paper:
\begin{center}
    \textit{What relaxation yields the best loss function for SGP ?}
\end{center}
We provide both empirical and theoretical evidence that $\GW(A,B)$ is the best among the three for our application. In Proposition~\ref{prop:gw-tightest}, we show that $\GW$ is a tighter relaxation than the alternatives. In Proposition~\ref{prop:gw-minima} we show that $\GW$ is a valid surrogate for $\ell$ given that $\GW(A,B)=0 \implies \ell(A,B)=0$ which is not the case for the alternatives. Finally, in Table \ref{tab:losses_comparisons} we demonstrate empirically that, within our framework, $\GW$ consistently yields superior results across datasets.

\begin{restatable}[$\GW$ is a closest relaxation]{proposition}{gwtightest}
\label{prop:gw-tightest}
If \(d\) is convex in $a$ and convex in $b$, then for all \(A,B\in[0,1]^{n\times n}\),
$$
    \max \left(\ell_a(A,B),\ \ell_b(A,B)\right)
    \;\le\;
    \GW(A,B)
    \;\le\;
    \ell(A,B).
$$
\end{restatable}
Note that related inequalities have appeared in prior work
\citep{aflalo2015convex}.
\begin{restatable}[$\GW$ has no spurious minima]{proposition}{gwspurious}
\label{prop:gw-minima}
If $d(a,b) = 0 \iff a=b$ and $0 \leq d(a,b)$, we have that
\begin{equation}
    \GW(A,B) = 0 \iff \ell(A,B) = 0
\end{equation}
On the contrary,  $\ell_a(A,B) = 0 \centernot\implies \ell(A,B) =0$ and $\ell_b(A,B) = 0 \centernot\implies \ell(A,B) =0$
\end{restatable}
Finally, note that $J(T; A,B)$ can be efficiently computed: given
any Bregman divergence \(d\) the tensor product
\(J(T;A,B)\) admits a low-rank factorization \citep{peyre2016gromov}, so it can
be evaluated  in \(\mathcal{O}(n^3)\) rather than the
\(\mathcal{O}(n^4)\) of the naive tensor contraction. The question of finding the optimal matching $T^* = \argmin_{T \in \Pi_n} J(T;A,B)$ remains and we discuss two alterative approaches below.

\paragraph{Numerical GW solver.}
We now turn to computing the loss, which requires minimizing \eqref{eq:gw-relaxation}
over \(\Pi_n\) for every training pair, a step that must be fast and fully
parallelizable on GPU. This can be achieved with the mirror-descent scheme of
\cite{xu2019scalable, peyre2016gromov}: from an initial plan
\(T_0\), each of the \(K_{\mathrm{out}}\) outer iterations solves
\begin{equation}
    T_{k+1}
    =
    \arg\min_{T\in\Pi_n}
    \langle \nabla_T J(T_k;A,A^\star), T\rangle
    + \tau\, \KL(T\,\|\,T_k).
\end{equation}
Linearizing \(J\) and expanding the KL term recasts this as an entropic OT
problem,
\begin{equation}
    T_{k+1}
    =
    \arg\min_{T\in\Pi_n}
    \langle C_k, T\rangle - \tau\, H(T),
    \qquad
    C_k = \nabla_T J(T_k;A,A^\star) - \tau\log T_k,
\end{equation}
which we solve with \(K_{\mathrm{in}}\) Sinkhorn iterations,
\(T_{k+1} = \operatorname{Sinkhorn}(C_k;\, \tau,\, K_{\mathrm{in}})\). The full
nested scheme is controlled by three parameters
\((\tau, K_{\mathrm{in}}, K_{\mathrm{out}})\); we defer the Sinkhorn details and
complete pseudo-code to Appendix~\ref{app:sink}. Wrapping this solver in the
loss as $\operatorname{Solver}(A,A^\star)$ gives
\begin{equation}
    \mathcal{L}_{\mathrm{solver}}(\theta)
    =
    \sum_{i=1}^N J\big(T_i;\, f_\theta(x_i),\, A_i^\star\big),
    \qquad
    T_i = \sg\big[\mathrm{Solver}(f_\theta(x_i), A_i^\star)\big],
\end{equation}where $\sg[\cdot]$
denotes the stop-gradient operator, justified by the envelope theorem \cite{blondel2024elements}: at the optimum the objective is stationary in \(T_i\), so backpropagating through the solver is
unnecessary, i.e. \(\nabla_\theta \GW(f_\theta(x), A^\star) = \nabla_\theta J(\sg[T^\star]; f_\theta(x), A^\star)\), where $T^\star = \arg \min_{T\in \Pi_n} J(T; f_\theta(x), A^\star)$.

\begin{figure}[t]
    \centering
    \includegraphics[width=\linewidth]{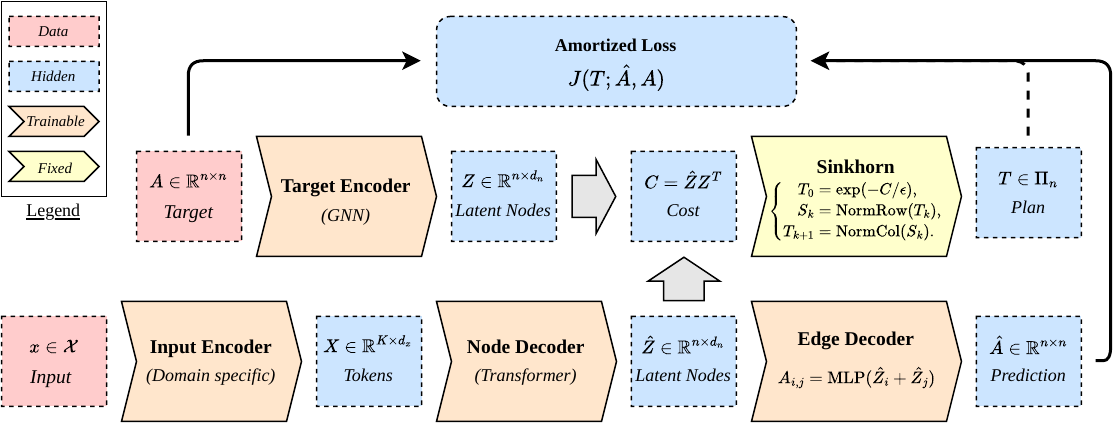}
    \caption{Amortized Supervised Graph Prediction Schema.}
    \label{fig:placeholder}
\end{figure}

\paragraph{Amortized Matcher.}
The solver above optimizes the matching objective directly but is costly: a
typical run, with \(100\)k steps at batch size \(128\), invokes it over
\(10^7\) times, so the inner Sinkhorn rollouts dominate  training time. Following the recent amortized matching introduced for graph autoencoding in
\citep{krzakala2025grale}, we replace the
per-step solver with a \emph{learned} matcher \(\mathrm{Matcher}_{\theta'}\)
that is trained jointly with \(f_\theta\) to predict the transport plan in a
single forward pass:
\begin{equation}
    \mathcal{L}_{\mathrm{matcher}}(\theta, \theta')
    =
    \sum_{i=1}^N J\big(T_i(\theta');\, f_\theta(x_i),\, A_i^\star\big)
    + \Omega(\theta'),
    \qquad
    T_i(\theta') = \mathrm{Matcher}_{\theta'}(x_i, A_i^\star).
    \label{eqloss}
\end{equation}
where $\Omega$ is regularization term discussed in equation~\eqref{eq:marginal_error}.
We parametrize the matcher with two embedding networks,
\(g_{\theta'}:\mathcal{X}\to\mathbb{R}^{n\times d_m}\) and
\(h_{\theta'}:\mathcal{G}\to\mathbb{R}^{n\times d_m}\), which map the input and
the target graph to node embeddings. Their inner product forms a rank-\(d_m\)
affinity matrix used for computing a matching with entropic regularized Optimal Transport:
\begin{equation}
    \mathrm{Matcher}_{\theta'}(x, A^\star)
    =
    \argmax_{T \in \Pi_n} \ \langle g_{\theta'}(x)\,h_{\theta'}(A^\star)^\top, T\rangle + \epsilon H(T),
    \label{eq:matcher_0}
\end{equation}
where the entropic smoothing parameter $\epsilon$ ensures the differentiability of the matcher. In practice, the plan is computed by the Sinkhorn Algorithm, 
\begin{equation}
    \mathrm{Matcher}_{\theta'}(x, A^\star)
    =
    \operatorname{Sinkhorn}\big(-g_{\theta'}(x)\,h_{\theta'}(A^\star)^\top;\ \epsilon,\ K_{\mathrm{m}}\big),
    \label{eq:matcher}
\end{equation}
where $K_{\mathrm{m}}$ is the number of steps of the algorithm. The matcher can be made equivariant to permutations of $A^*$ but this might not be desirable as we discuss in details in Appendix \ref{app:lpe}.

Two properties are desirable here: a small \(K_{\mathrm{m}}\), since each
Sinkhorn rollout is expensive, and a small \(\epsilon\), so that
\(T_i(\theta')\) is close to a permutation. These conflict: at small
\(\epsilon\), the few Sinkhorn iterations we can afford do not converge, meaning that the rows and columns sums of \(T_i(\theta')\) deviate from \(\mathbf{1}_n\). Rather than run Sinkhorn to convergence, we propose to
keep \(K_{\mathrm{m}}\) small and penalize these marginal deviations directly:
\begin{equation}
\label{eq:marginal_error}
    \Omega(\theta')
    =
    \sum_i
    \KL\big(T_i(\theta')\mathbf{1}_n \,\|\, \mathbf{1}_n\big)
    + \KL\big(T_i(\theta')^\top\mathbf{1}_n \,\|\, \mathbf{1}_n\big).
\end{equation}
This penalty connects to unbalanced optimal transport \citep{chizat2018scaling} which was already used in an amortized framework \cite{mazelet2026unsupervised}. While the penalty is the same, our intention differ as we do not aim to relax the marginal constraints but to prevent the model to use the low Sinkhorn iteration budget to escape the loss (e.g. with $T=0$).
The full amortized scheme is summarized in
Figure~\ref{fig:placeholder}.

\section{Numerical Experiments}

\subsection{Experimental setting.}

\paragraph{Datasets.} We evaluate on three tasks. \textbf{Coloring} \citep{krzakala2024any2graph} is a
synthetic benchmark of connected graphs labeled with a valid four-coloring where the input is an image representing the graph. We construct two molecular tasks. \textbf{Fingerprint2Molecule} recovers a molecular graph from its binary structural fingerprint, following \citep{krzakala2024any2graph}, with molecules drawn from PubChem \cite{kim2016pubchem}. \textbf{MS2Scaffold} predicts a molecular scaffold graph from a tandem mass spectrum, using the MassSpecGym benchmark \citep{bushuiev2024massspecgym} with the formula-based split recommended by \citep{krzakala2026msalign}. We propose the latter as an intermediate step toward \emph{de novo}
metabolite identification where the full molecule should be reconstructed which remains extremely challenging
\citep{bushuiev2024massspecgym}. For instance, even state-of-the-art models such as MetGenX \citep{wang2026structure} (\emph{Nature Comm.} 2026) reaches only \(2.50\%\) top-1 on the MassSpecGym benchmark \citep{bushuiev2024massspecgym}.\looseness=-1

\paragraph{Model architecture.} Following Any2Graph \cite{krzakala2024any2graph}, we parametrize the graph prediction model $f_\theta$ with three components:
\begin{itemize}
   \setlength\itemsep{0.3em}
   \setlength{\parskip}{0pt}
   \setlength{\topsep}{-5pt}
   \item A domain specific Input Encoder $E_{\theta_1}:\mathcal{X}\mapsto R^{K\times d}$ that encode the input with $K$ latent vectors. The model used for each dataset is detailed in Appendix \ref{app:exp}.
   \item A node decoder $D_{\theta_2}:R^{K\times d}\mapsto R^{n \times d}$ that predicts some latent representations of the target graph's nodes. This block is implemented with a Transformer \cite{vaswani2017attention}
   \item A graph prediction head $H_{\theta_3}:R^{n\times d}\mapsto \mathcal{G}$ parametrized with MLPs as in \cite{krzakala2024any2graph}.
\end{itemize}
Overall, the graph prediction model is $f_\theta = H_{\theta_3}\circ D_{\theta_2} \circ E_{\theta_1}$. Finally, the node embeddings used by the matcher \eqref{eq:matcher},  are defined as follows:
\begin{itemize}
   \setlength\itemsep{0.3em}
   \setlength{\parskip}{0pt}
   \setlength{\topsep}{-5pt}
   
   \item For the predicted graph node embeddings, we reuse the first part of  $f_\theta$ that is  $g_{\theta'}= D_{\theta_2'} \circ E_{\theta_1'}$ and we apply weight sharing $(\theta_1,\theta_2) = (\theta_1',\theta_2')$.
   \item The target graph node embeddings are extracted by a 3 layers graph neural network $h_{\theta'}$ \cite{xu2018powerful} with Laplacian Positionnal Encoding as discussed in Appendix \ref{app:lpe}.
\end{itemize}
 Throughout all experiments, the loss used for $d(a,b)$ is the cross-entropy loss. More details on the architecture and hyperparameters are provided in Appendix~\ref{app:exp}. Code is available at
\href{https://github.com/FedericoMendez/amortized-graph-prediction}{ \texttt{https://github.com/FedericoMendez/amortized-graph-prediction}\;\faGithub}.
\begin{table}[t]
\caption{Comparison of the different SGP strategies across the graph prediction tasks.}
\vspace{0.1cm}
\label{tab:main_results}
\footnotesize
\centering
\setlength{\tabcolsep}{4pt}
\begin{tabular}{l||cc|cc|cc|cc}
\toprule
& \multicolumn{2}{c|}{\textbf{Coloring 10}}
& \multicolumn{2}{c|}{\textbf{Coloring 20}}
& \multicolumn{2}{c|}{\textbf{MS2Scaffold}}
& \multicolumn{2}{c}{\textbf{Fingerprint2Mol}} \\
\textbf{Model}
& \textsc{Edit} $\downarrow$ & \textsc{GI Acc.} $\uparrow$
& \textsc{Edit} $\downarrow$ & \textsc{GI Acc.} $\uparrow$
& \textsc{Edit} $\downarrow$ & \textsc{GI Acc.} $\uparrow$
& \textsc{Edit} $\downarrow$ & \textsc{GI Acc.} $\uparrow$ \\
\midrule
\textsc{FGWBary}
 & 6.73 & 1.00
 & -- & --
 & -- & --
 & -- & -- \\
\textsc{Relationformer}
 & 5.47 & 18.14
 &  17.28 & 3.58
 & 32.72 & 0.00
 & 25.20 & 0.00 \\
\textsc{Any2Graph}
 & 2.82 & 31.79
 & 16.64 & 18.92
 & 19.11 & 5.61
 & 17.64 & 0.001 \\
\textsc{Any2Graph + Matcher}
 & \textbf{2.30} & \textbf{45.46}
 & \textbf{02.64} & \textbf{46.47}
 & \textbf{12.17} & \textbf{29.48}
 & \textbf{7.01} & \textbf{15.79} \\
\bottomrule
\end{tabular}
\vspace{-0.3cm}
\end{table}

\subsection{Results}

\begin{table}[t]
\caption{Effect of the continuous graph-matching relaxation used to train the
learned matcher. The Gromov-Wasserstein objective consistently yields the best
edit distance and exact-reconstruction accuracy.}
\label{tab:losses_comparisons}
\footnotesize
\centering
\setlength{\tabcolsep}{4pt}
\begin{tabular}{l||cc|cc|cc}
\toprule
& \multicolumn{2}{c|}{\textbf{Coloring 20}}
& \multicolumn{2}{c|}{\textbf{MS2Scaffold}}
& \multicolumn{2}{c}{\textbf{Fingerprint2Mol}} \\
\textbf{Loss}
& \textsc{Edit} $\downarrow$ & \textsc{GI Acc.} $\uparrow$
& \textsc{Edit} $\downarrow$ & \textsc{GI Acc.} $\uparrow$
& \textsc{Edit} $\downarrow$ & \textsc{GI Acc.} $\uparrow$ \\
\midrule
$J(T; f_\theta(x), A^*)$
& \textbf{2.64} & \textbf{46.47}
& \textbf{12.17} & \textbf{29.48}
& \textbf{7.01} & \textbf{15.79} \\
$J_a(T; f_\theta(x), A^*)$& 19.36 & 7.35
& 17.98 & 3.15
& 15.56 & 0.42 \\
$J_b(T; f_\theta(x), A^*)$ & 22.42 & 0.61
& 20.29 & 0.23
& 18.71 & 0.14 \\
 $J_a(T; A^*,  f_\theta(x))$ & 25.84 & 0.42 & 24.39 & 0.00 & 29.49 & 0.01 \\
$J_b(T; A^*, f_\theta(x)$ & 13.61 & 13.51 & 17.19 & 4.32 & 19.47 & 0.12 \\
\bottomrule
\end{tabular}
\vspace{-0.3cm}
\end{table}

\paragraph{Predictive performance.}
We report edit distance (\(\downarrow\)) and graph-isomorphism accuracy
(\textsc{GI Acc.}\(\uparrow\)), i.e.\ the fraction of exactly reconstructed
graphs, at the best validation epoch. We compare against three SGP baselines:
\textsc{FGWBary}~\cite{brogat2022learning},
\textsc{Relationformer}~\cite{shit2022relationformer}, and the original
\textsc{Any2Graph}~\cite{krzakala2024any2graph}, without the learnable matcher.
Table~\ref{tab:main_results} shows that adding the matcher improves every task.
The gain is especially pronounced on the molecular benchmarks: relative to
\textsc{Any2Graph}, it reduces edit distance from \(19.11\) to \(12.17\) on
\textsc{MS2Scaffold} and from \(17.64\) to \(7.01\) on
\textsc{Fingerprint2Mol}, while increasing exact-reconstruction accuracy from
\(5.61\%\) to \(29.48\%\) and from \(0.001\%\) to \(15.79\%\), respectively.

\begin{figure}[t!]
    \centering
    \includegraphics[width=0.48\textwidth]{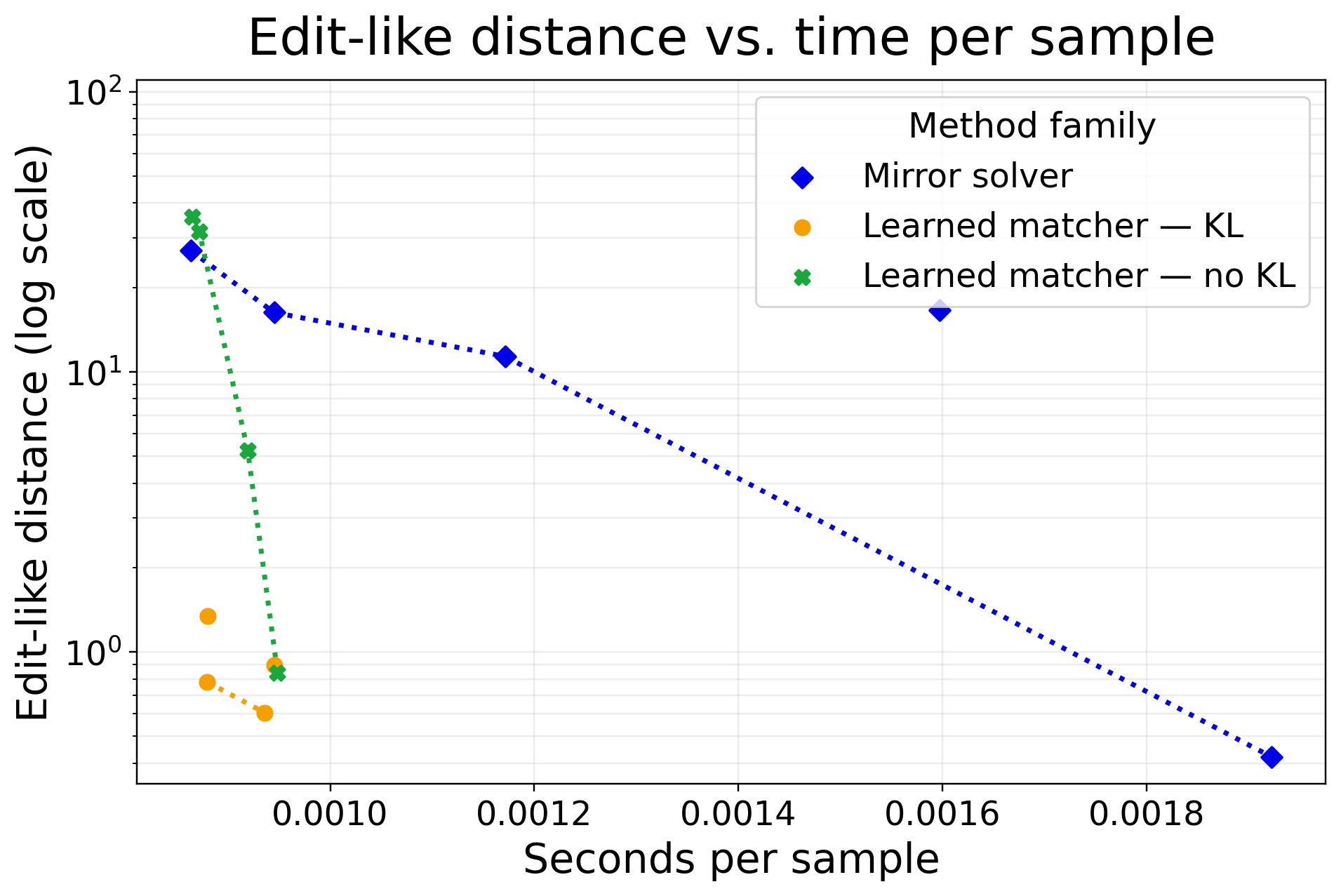}
    \hfill
    \includegraphics[width=0.48\textwidth]{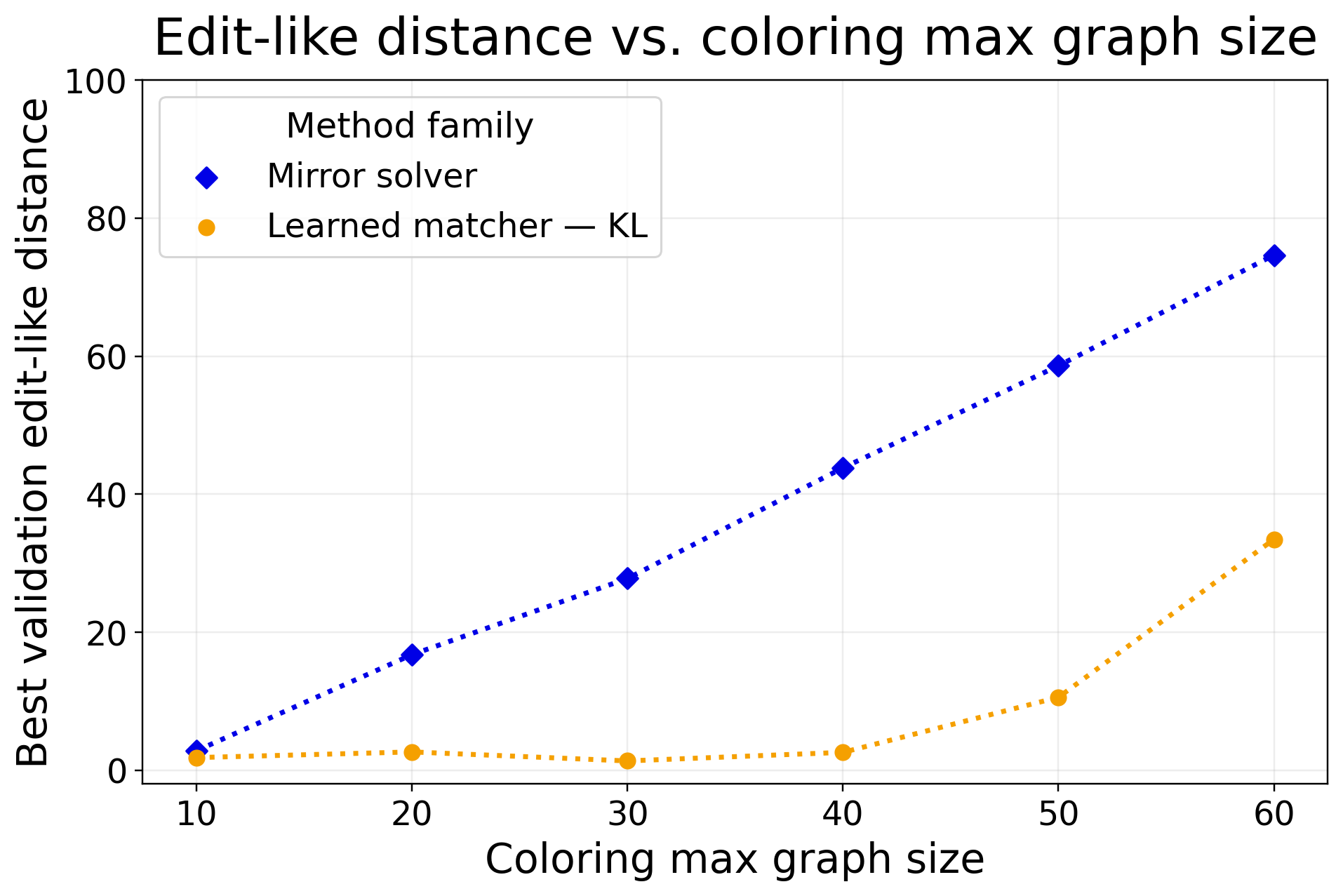}
    \caption{Left: We report the prediction performances (Edit Distance $\downarrow$) against the model throughput (Seconds/Sample $\downarrow$) for a variety of models and hyperparameter choices. Right: We compare the performances (Edit Distance $\downarrow$) of the matcher- and solver-based approach for datasets of increasing maximum graph size. The COLORING synthetic benchmark is used in both figures \cite{krzakala2024any2graph}.}
    \label{fig:main}
    \vspace{-0.2cm}
\end{figure}

\paragraph{Efficiency and scaling.}
Figure~\ref{fig:main} compares the learned matcher with the numerical mirror
solver on \textsc{Coloring}. In the left panel, the KL-regularized matcher forms
the favorable accuracy-runtime frontier: at comparable per-sample cost it
achieves a markedly smaller edit-like distance, while the solver attains a
slightly lower distance only at approximately twice the latency. Removing the
KL penalty worsens this trade-off. The right plot shows that the advantage
widens with graph size: the solver's error rises steadily, whereas the learned
matcher remains accurate through graphs of size \(40\) and retains a large
advantage at sizes \(50\) and \(60\).

\paragraph{Choice of relaxation.} We now return to the first question raised in
Section~\ref{sec:efficient}: \emph{which relaxation yields the best loss
function?} Our theoretical results identified \(\GW\) as a tighter and more
faithful surrogate than \(\ell_a\) or \(\ell_b\), and
Table~\ref{tab:losses_comparisons} confirms this empirically. Replacing \(\GW\)
with either alternative degrades both metrics on all three benchmarks. The
performance gap between \(\GW\) and the two alternative relaxations is not
limited to exact reconstruction accuracy: both \(\ell_a\) and \(\ell_b\) also
produce substantially larger edit distances. This suggests that these weaker
relaxations provide a less informative learning signal even when predictions
are not exactly correct.

\section{Conclusion}
We introduced a learned matcher that amortizes graph alignment in supervised
graph prediction, avoiding an iterative solver for every training sample. Our
theoretical and empirical results support the Gromov-Wasserstein objective over
the alternative bistochastic relaxations considered. Across synthetic and molecular tasks,
the proposed approach improves reconstruction performance, provides a favorable
accuracy-runtime trade-off, and scales better with graph size. These results
show that amortized alignment is an effective way to reduce the computational
cost of optimal transport based graph prediction.

\section*{Acknowledgments and Disclosure of Funding}
The study was funded by French National Research Agency (ANR) through PEPR IA FOUNDRY (ANR-23-PEIA-0003), e-Lucid (ANR-25-
TSIA-0002-01), the France 2030 program under the MacLeOD project (ANR-25-PEIA-0005), and from the European Union’s Horizon Europe research and innovation programme under grant agreement 101120237 (ELIAS). This work benefited from Hi! PARIS and State funding managed by the French National Research Agency (ANR) under the France 2030 program, reference ANR-23-IACL-0005. 
The second and third authors received PhD scholarships from Institut Polytechnique de Paris.
\newpage
\FloatBarrier
\bibliographystyle{unsrt} 
\bibliography{references}
\newpage
\appendix
\section{Algorithms}
\label{app:sink}

\subsection{Reminder of Sinkhorn}
Given a cost matrix \(C\in\mathbb{R}^{n\times n}\) and regularization
\(\tau>0\), the Sinkhorn algorithm \citep{cuturi2013sinkhorn} solves the
entropy-regularized OT problem
\(\min_{T\in\Pi_n}\langle C,T\rangle - \tau H(T)\)
by alternating marginal projections on the Gibbs kernel \(K=\exp(-C/\tau)\). Each iteration costs \(\mathcal{O}(n^2)\) and
uses only matrix and vector products, making it fully parallelizable on GPU an differentiable by unrolling \cite{genevay2018learning}. Note that in the original Sinkhorn algorithm one has to make an arbitrary choice: what is the marginal that will get exactly enforced in the last step, row or column ? To remove this ambiguity we add a final half step that average the last 2 updates (Algorithm~\ref{alg:sinkhorn}).

\begin{algorithm}[H]
\caption{\( \operatorname{Sinkhorn}(C;\tau,K_{\mathrm{in}})\)}
\label{alg:sinkhorn}
\begin{algorithmic}[1]
\State \(K \gets \exp(-C/\tau)\), \quad \(v \gets \mathbf{1}_n\)
\For{\(k = 1\) to \(K_{\mathrm{in}}\)}
    \State \(u \gets \mathbf{1}_n \,/\, (K v)\) \Comment{Row normalization}
    \State \(v \gets \mathbf{1}_n \,/\, (K^\top u)\) \Comment{Column normalization}
\EndFor
\State \(u' \gets \mathbf{1}_n \,/\, (K v)\) \Comment{Extra Row normalization}
\State \(u \gets \frac{u+u'}{2}\) \Comment{Average the last two steps}
\State \Return \(T = \operatorname{diag}(u)\,K\,\operatorname{diag}(v)\)
\end{algorithmic}
\end{algorithm}

\subsection{Mirror Descent Solver}
Recall that \(J(T;A,B) = \sum_{i,j,k,l} d(A_{ik},B_{jl})\, T_{ij}T_{kl}\) is the
quadratic matching cost between adjacency matrices \(A\) and \(B\). In Algorithm \ref{alg:solver}, we propose to use mirror descent to find the optimal transport $T^* = \argmin_{T\in \Pi_n} J(T;A,B)$ \cite{nemirovskij1983problem}. By leveraging the Kullback-Leibler geometry as proposed in \cite{xu2019scalable}, each iterations rewrites as an entropic-regularized OT problem which can we solved with the Sinkhorn algorithm.

\begin{algorithm}[H]
\caption{Mirror-descent \(\mathrm{Solver}(A, A^\star)\)}
\label{alg:solver}
\begin{algorithmic}[1]
\Require adjacency matrices \(A, A^\star\); parameters \((\tau, K_{\mathrm{in}}, K_{\mathrm{out}})\)
\State initialize \(T_0 \gets \tfrac{1}{n}\mathbf{1}_n\mathbf{1}_n^\top\) \Comment{uniform plan in \(\Pi_n\)}
\For{\(k = 0\) to \(K_{\mathrm{out}}-1\)}
    \State \(C_k \gets \nabla_T J(T_k; A, A^\star) - \tau \log T_k\) \Comment{\(\mathcal{O}(n^3)\) via factorization \cite{krzakala2024any2graph}}
    \State \(T_{k+1} \gets \operatorname{Sinkhorn}(C_k;\, \tau,\, K_{\mathrm{in}})\)
\EndFor
\State \Return \(T_{K_{\mathrm{out}}}\)
\end{algorithmic}
\end{algorithm}

\subsection{Amortized Implementation}
In the proposed amortized implementation, the naive loss
\begin{equation}
    \mathcal{L}_{\mathrm{naive}}(\theta) =
    \sum_i \min_{T\in \Pi_n} J\big(T;\, f_\theta(x_i),\, A_i^\star\big),
\end{equation}
is replaced by 
\begin{equation}
    \mathcal{L}_{\mathrm{matcher}}(\theta, \theta') =
    \sum_i J\big(T_i(\theta');\, f_\theta(x_i),\, A_i^\star\big)
    + \Omega(\theta'),
\end{equation}
where $T_i(\theta') = \operatorname{Sinkhorn}\big(g_{\theta'}(x)\,h_{\theta'}(A^\star)^\top;\ \epsilon,\ K_{\mathrm{in}}\big).$ is the predicted matching, parametrized as the optimal (entropic) matching between some learnable node features and $\Omega(\theta')$ is the marginal violation defined in \ref{eq:marginal_error}. In practice, we recommend that the target predictor $f_\theta$ and the node target predictor $g_\theta'$ share a common backbone $\phi_\theta$. This approach is summarized in Algorithm \ref{alg:amortized}.

\begin{algorithm}[H]
\caption{Amortized training step (matcher \(\theta'\) + predictor \(\theta\))}
\label{alg:amortized}
\begin{algorithmic}[1]
\Require batch \(\{(x_i, A_i^\star)\}_{i=1}^B\); parameters \((\epsilon, K_{\mathrm{in}})\); networks \(g_{\theta'}, h_{\theta'}, f_\theta\)
\For{\(i = 1\) to \(B\)}
    \State \(\hat{Z}_i \gets \phi_\theta(x_i)\) \Comment{backbone}
    \State \(\hat{A}_i \gets f_\theta(\hat{Z}_i)\) \Comment{predicted graph}
    \State \(C_i \gets g_{\theta'}(\hat{Z}_i)\, h_{\theta'}(A_i^\star)^\top\) \Comment{rank-\(d_m\) matching cost}
    \State \(T_i \gets \operatorname{Sinkhorn}(C_i;\, \epsilon,\, K_{\mathrm{in}})\) 
    \Comment{predicted plan}
    \State \(\Omega_i \gets KL\big(T_i\mathbf{1}_n \,\|\, \mathbf{1}_n\big) + KL\big(T_i^T\mathbf{1}_n \,\|\, \mathbf{1}_n\big)  \)  
    \Comment{marginal penalization}
\EndFor
\State \(\mathcal{L}_{\mathrm{matcher}} \gets \sum_i J(T_i; \hat{A}_i, A_i^\star) + \Omega_i\)
\State update \(\theta, \theta'\) by a gradient step on \(\mathcal{L}_{\mathrm{matcher}}\) \Comment{backprop through both \(J\) and \(\operatorname{Sinkhorn}\)}
\State \Return updated \(\theta, \theta'\)
\end{algorithmic}
\end{algorithm}
\newpage
\section{Technical details}

\subsection{Extensions to labeled graphs of arbitrary sizes}
\label{extension_labeled}
For completeness, we explain how the framework extends to variable-size graphs
with node and edge features, following Any2Graph~\citep{krzakala2024any2graph}.
Let a prediction be
$\hat y=(\hat h,\hat F,\hat A)$, where $\hat h\in[0,1]^n$ is a soft node mask
(equivalently, a vector of node-existence probabilities), $\hat F_i$ is the feature of node $i$, and $\hat A$ is the predicted
adjacency matrix.  The target
$y^\star=(h^\star,F^\star,A^\star)$ is padded to the same maximum size $n$;
$h_j^\star=1$ exactly for its $m=\lVert h^\star\rVert_1$ genuine nodes.  The
Partially-Masked Fused Gromov-Wasserstein loss is

\begin{align}
 \operatorname{PMFGW}(\hat y,y^\star)
 =\min_{T\in\Pi_n}\;&
 \frac{\alpha_h}{n}\sum_{i,j}T_{ij}
 d_h(\hat h_i,h_j^\star)
 +\frac{\alpha_F}{m}\sum_{i,j}T_{ij}h_j^\star
 d_F(\hat F_i,F_j^\star) \notag\\
 &+\frac{\alpha_A}{m^2}\sum_{i,j,k,l}T_{ij}T_{kl}
 h_j^\star h_l^\star
 d_A(\hat A_{ik},A_{jl}^\star).
 \label{eq:pmfgw-attributed}
\end{align}
The first term learns the graph size, the second aligns node features, and the
third aligns the edges. The weights
$\alpha=(\alpha_h,\alpha_F,\alpha_A)$ control the relative importance
of these three objectives. For a plain fixed-size unlabeled graph, the third term is exactly
the GW objective used in the main text (and the first two terms may be omitted). For an edge-attributed graph extension, we refer to \citep{yang2024exploiting}.

\subsection{Equivariant matcher or symmetry breaking matcher}
\label{app:lpe}

Recall the definition of the matcher:
\begin{equation}
    \mathrm{Matcher}_{\theta'}(x, A^\star)
    =
    \argmax_{T \in \Pi_n} \ \langle g_{\theta'}(x)\,h_{\theta'}(A^\star)^\top, T\rangle + \epsilon H(T),
    \label{eq:matcher_0}
\end{equation}
where \(g_{\theta'}:\mathcal{X}\to\mathbb{R}^{n\times d_m}\) and
\(h_{\theta'}:\mathcal{G}\to\mathbb{R}^{n\times d_m}\). The matcher is expected to approximate the optimal transport plan between the prediction $f(x)$ and the target $A^\star$. Denoting:
\begin{equation}
    T(f(x),A^\star) = \argmin_{T\in \Pi_n} J(T; f(x), A^\star),
\end{equation}
we expect 
\begin{equation}
    \mathrm{Matcher}_{\theta'}(x, A^\star) \approx T(f(x),A^\star).
\end{equation}
Interestingly, the optimal transport plan is \emph{permutation equivariant}, formally:
\begin{equation}
    \forall P \in \Sigma_n, \quad T(f(x),PA^\star P^T) = T(f(x),A^\star)P^T.
\end{equation}
Consequently, it might seem natural to enforce a similar property in the matcher. This is easy to achieve: it suffices that $h_{\theta'}$ be permutation equivariant for the matcher to inherit the property, as stated in the following proposition.

\begin{proposition}[Permutation Equivariant Matcher]
If the target node encoder $h_{\theta'}$ is permutation equivariant, 
\begin{equation}
    \forall P \in \Sigma_n, \quad h_{\theta'}(PA^\star P^T) = Ph_{\theta'}(A^\star),
\end{equation}
Then, the matcher is permutation equivariant,
\begin{equation}
    \forall P \in \Sigma_n, \quad \mathrm{Matcher}_{\theta'}(f(x),PA^\star P^T) = \mathrm{Matcher}_{\theta'}(f(x),A^\star)P^T.
\end{equation}
\end{proposition}
\begin{proof}
Assume that $h_{\theta'}$ is permutation equivariant. Then, for any $P \in \Sigma_n$, we have 
\begin{align*}
    \mathrm{Matcher}_{\theta'}(f(x),PA^\star P^T) &= \argmax_{T \in \Pi_n} \ \langle g_{\theta'}(x)\,h_{\theta'}(PA^\star P^T)^\top, T\rangle + \epsilon H(T) \\
    &= \argmax_{T \in \Pi_n} \ \langle g_{\theta'}(x)\,h_{\theta'}(A^\star )^\top P^T, T\rangle + \epsilon H(T)\\
    &= \argmax_{T \in \Pi_n} \ \langle g_{\theta'}(x)\,h_{\theta'}(A^\star )^\top , TP\rangle + \epsilon H(TP)
\end{align*}
where we used that $H(T) = H(TP)$. The change of variable $T'=TP$ concludes the proof.
\end{proof}
Yet, previous works have highlighted that permutation equivariant models might be limited in the kind of mapping they can perform \cite{ICLR2025_c7138635}. In the context of the matcher, this was a significant limitation in previous works (see Appendix B.3 from GRALE \cite{krzakala2025grale}). Formally, these limitations arise when the automorphism group of $A^*$ is non-trivial as detailled in the next proposition.
\begin{proposition}[Limitations of a permutation invariant matcher]
\label{prop:automorphism_matcher}
Assume that $P\in \Sigma_n$ is the automorphism group of $A^*$ i.e. $PA^*P^T=A^*$. Then, for any permutation equivariant matcher we have,
\begin{equation}
    \forall x\in\mathcal{X}, \quad \mathrm{Matcher}_{\theta'}(x, A^\star)P^T
    =
    \mathrm{Matcher}_{\theta'}(x, A^\star)
    \label{eq:matcher_permutation}
\end{equation}
In particular, if $A^*$ is vertex-transitive, the output of the matcher is trivial,
\begin{equation}
    \forall x\in\mathcal{X}, \quad \mathrm{Matcher}_{\theta'}(x, A^\star)
    =
    \frac{1}{n}\mathbf{1}\mathbf{1}^T
\end{equation}
This would apply for instance, to the Benzene molecule in Figure~\ref{fig:examples}.
\end{proposition}
\begin{proof}
The first equation is a direct application of the permutation equivariance assumption. If we assume that $A^*$ is vertex-transitive, then for any $i,j \in [1,n]$, there exists $P \in \Sigma_n$ such that 
$$PA^*P^T = A^* \quad \text{and} \quad P_{i,j} =1$$
Applying \eqref{eq:matcher_permutation} we get in particular that the i-th and j-th columns of $T=\mathrm{Matcher}_{\theta'}(x, A^\star)$ are identical. Therefore all the columns of $T$ are identical. Finally the marginal condition $T \mathbf{1}_n = \mathbf{1}_n $ implies that columns are constant to $\frac{1}{n}\mathbf{1}_n$ which concludes the proof.
\end{proof}
Proposition \ref{prop:automorphism_matcher} shows that a permutation equivariant matcher is unable to differentiate the symmetric nodes of a graph. Thus, if the automorphism group of $A^*$ is not trivial, the matcher is forced to output a fuzzy one-to-many matching matrix.

We address this limitation by injecting some symmetry breaking Laplacian Positionnal Encoding (LPE)\footnote{LPE is not exactly equivariant due to sign and repeated eigenvalues ambiguities. These ties are broken arbitrarily which breaks the symmetries of the graph.} to the graph neural network $h_{\theta'}$, future work could consider more subtle positionnal encoding such that proposed in \cite{ICLR2025_c7138635}.

\newpage
\section{Proofs}
\label{app:proofs}

\equivmatching*
\begin{proof}
Fix \(P\in\Sigma_n\). Since \(P\) is a permutation matrix, each of its rows and
columns contains exactly one entry equal to \(1\) and all others equal to \(0\).
Consequently, for any vector \(x\in\mathbb{R}^n\) and any index \(l\),
\begin{equation}
    \sum_{k=1}^n P_{kl}\, x_k = x_{\pi(l)},
    \label{eq:onehot}
\end{equation}
where \(\pi\) is the permutation associated with \(P\) (i.e.\ \(P_{kl}=1\) iff
\(k=\pi(l)\)). In particular, applying \eqref{eq:onehot} entrywise inside the
cost \(d\) is exact, since only one term of the sum is nonzero:
\begin{equation}
    \sum_{k=1}^n P_{kl}\, d(A_{ik}, B_{jl})
    = d\big([AP]_{il},\, B_{jl}\big).
    \label{eq:collapse}
\end{equation}

\emph{Equivalence \(J = J_a\).}
Starting from \eqref{eq:qap} and summing over \(k\) using \eqref{eq:collapse},
then over \(j\) in the same way,
\begin{align*}
    J(P;A,B)
    &= \sum_{i,j,l} \Big(\sum_k P_{kl}\, d(A_{ik},B_{jl})\Big) P_{ij}
     = \sum_{i,j,l} d\big([AP]_{il}, B_{jl}\big)\, P_{ij} \\
    &= \sum_{i,l} d\big([AP]_{il},\, [PB]_{il}\big)
     = J_a(P;A,B),
\end{align*}
where the third equality collapses the sum over \(j\) exactly as in
\eqref{eq:collapse}, now acting on the second argument of \(d\).

\emph{Equivalence \(J = J_b\).}
Applying \eqref{eq:onehot} in both the \(j\) and \(l\) sums of \eqref{eq:qap}
simultaneously,
\[
    J(P;A,B)
    = \sum_{i,k} d\big(A_{ik},\, [PBP^\top]_{ik}\big)
    = J_b(P;A,B).
\]

Both reductions rely only on the one-hot structure of \(P\), so they hold for
every \(P\in\Sigma_n\).
\end{proof}

\begin{proposition}[Non-equivalent bistochastic objectives.] \label{prop:not-equivalent}Assume that $d(a,b)=(a-b)^2$. Then for any $T\in\Pi_n$, we have
\begin{align*}
J(T;A,B) &= \|A\|_F^2 + \|B\|_F^2 - 2 \langle AT, TB \rangle.\\
J_a(T;A,B) &= \|AT\|_F^2 + \|TB\|_F^2 - 2 \langle AT, TB \rangle.\\
J_b(T;A,B) &= \|A\|_F^2 + \|TBT^\top\|_F^2 - 2 \langle AT, TB \rangle.
\end{align*}
In particular, there exist $A,B \in [0,1]^{n\times n}$ and $T\in\Pi_n$ such that:
\begin{align*}
J(T;A,B) = 2 \quad
J_b(T;A,B) = 1\quad
J_a(T;A,B) = 0.
\end{align*}
\end{proposition}
\begin{proof}
The provide the proof for the matrix formulation of $J(T;A,B)$, the two other cases follows from similar computations. Starting from the definition:
\begin{align*}
    &J(T;A,B)
    =
    \sum_{i,j,k,l=1}^n
    (A_{ik} - B_{jl})^2 T_{ij}T_{kl}\\
    &= \sum_{i,k}^n A_{ik}^2 \left(\sum_{jl}^n T_{ij}T_{kl}\right)
    + \sum_{jl}^n B_{jl}^2 \left(\sum_{i,k}^n T_{ij}T_{kl}\right)    -2\sum_{i,l=1}^n \left(\sum_{k=1}^n A_{ik} T_{kl} \right) \left(\sum_{j=1}^n T_{ij} B_{jl}\right)\\
    &= \sum_{i,k}^n A_{ik}^2
    + \sum_{jl}^n B_{jl}^2     -2 \sum_{i,l=1}^n [AT]_{il} [TB]_{il}\\
    &= \|A\|_F^2 + \|B\|_F^2 - 2 \langle AT, TB \rangle.
\end{align*}
Then, setting 
\begin{equation}
A = B = \begin{bmatrix}1 & 0 \\ 0 & 1 \end{bmatrix} \quad \text{and} \quad T = \begin{bmatrix} 0.5 & 0.5 \\ 0.5 & 0.5 \end{bmatrix}.
\end{equation}
Yields the special case.
\end{proof}

\gwtightest*
\begin{proof}
Recall \(\GW(A,B)=\min_{T\in\Pi_n} J(T;A,B)\),
\(\ell_a(A,B)=\min_{T\in\Pi_n} J_a(T;A,B)\),
\(\ell_b(A,B)=\min_{T\in\Pi_n} J_b(T;A,B)\), and
\(\ell(A,B)=\min_{P\in\Sigma_n} J(P;A,B)\).

\emph{Upper bound \(\GW \le \ell\).}
Since \(\Sigma_n\subset\Pi_n\), minimizing \(J\) over the larger set \(\Pi_n\)
can only decrease:
\[
    \GW(A,B) = \min_{T\in\Pi_n} J(T;A,B)
    \;\le\; \min_{P\in\Sigma_n} J(P;A,B) = \ell(A,B).
\]

\emph{Lower bound \(\ell_a \le \GW\).}
Now assume \(d\) is convex in both arguments. Fix \(T \in \Pi_n\) and \((i,l)\). The row and column constraints give \(\sum_k T_{kl}=1\) and \(\sum_j T_{ij}=1\), so applying Jensen's inequality to each argument of \(d\) in turn, 
\[
    d\big([AT]_{il},[TB]_{il}\big)
    = d\Big(\textstyle\sum_k A_{ik}T_{kl},\, \sum_j T_{ij}B_{jl}\Big)
    \;\le\;
    \sum_{j,k} T_{ij}T_{kl}\, d\big(A_{ik},B_{jl}\big).
\]
Summing over \((i,l)\) gives \(J_a(T;A,B)\le J(T;A,B)\), and minimizing over
\(T\in\Pi_n\) gives \(\ell_a(A,B)\le\GW(A,B)\).

\emph{Lower bound \(\ell_b \le \GW\).}
Fix \(T\in\Pi_n\) and \((i,k)\). Since \(\sum_{j,l} T_{ij}T_{kl}=1\) and convexity of
\(d\) in its second argument gives, by Jensen,
\[
    d\big(A_{ik}, [TBT^\top]_{ik}\big)
    = d\Big(A_{ik}, \textstyle\sum_{j,l} T_{ij}B_{jl}T_{kl}\Big)
    \;\le\;
    \sum_{j,l} T_{ij}T_{kl}\, d(A_{ik}, B_{jl}).
\]
Summing over \((i,k)\) gives \(J_b(T;A,B)\le J(T;A,B)\); minimizing over
\(T\in\Pi_n\) yields \(\ell_b(A,B)\le\GW(A,B)\).
\end{proof}

\gwspurious*
\begin{proof} We first prove that $\GW$ has no spurious minima and then exhibit counter examples for $\ell_a$ and $\ell_b$.\\
\textbf{GW has no spurious minima:} Let us assume that $\GW(A,B) = 0$. Then there exist $T\in \Pi_n$ such that:
\begin{align*}
    J(T;A,B) = \sum_{i,j,k,l=1}^n d(A_{ik}, B_{jl}) T_{ij} T_{kl} = 0
\end{align*}
Writing $T = \sum_{a} \lambda_a P^a$ the Birkhoff decomposition of $T$ as a convex combination of permutation matrices $P^a \in \Sigma_n$ \cite{birkhoff1946three}, we have
\begin{equation*}
\sum_{a,b}\sum_{i,j,k,l=1}^n d(A_{ik}, B_{jl}) \lambda_a \lambda_b P^a_{ij} P^b_{kl} = 0.
\end{equation*}
Without loss of generality, we can assume that all the $\lambda_a$ are strictly positive. The nonnegativity of $d$ then implies that, for all $a,b$,
\begin{equation*}
\sum_{i,j,k,l=1}^n d(A_{ik}, B_{jl}) P^a_{ij} P^b_{kl} = 0.
\end{equation*}
In particular, setting $a=b$, we have
\begin{equation*}
\sum_{i,j,k,l=1}^n d(A_{ik}, B_{jl}) P^a_{ij} P^a_{kl} = 0.
\end{equation*}
Since $P^a$ is a permutation matrix, this is equivalent to
\begin{equation*}
\sum_{i,k=1}^n d(A_{ik}, [P^a B (P^a)^\top]_{ik}) = 0,
\end{equation*}
which yields $A=P^aB(P^a)^\top$ by the separation property of $d$. Conversely, if $A=PBP^\top$ for some $P\in\Sigma_{n}$, then $J(P;A,B)=0$, proving the reverse implication.

\textbf{Counter example for $\ell_a$:} We provide a counter example for the case of the quadratic error $d(a,b)=(a-b)^2$. Assume that $A$ and $B$ are symmetric and satisfy $A\mathbf{1}_n=B\mathbf{1}_n=k\mathbf{1}_n$. For instance this apply when $A$ is  the adjacency matrices of the cycle graph $C_6$ and $B$ that of the disjoint union of two triangles. In that case $A$ anb $B$ are not isomorphic. Yet, setting transport plan $T=\frac{1}{n}\mathbf{1}_n\mathbf{1}_n^\top$, we obtain
\begin{align*}
AT
&= \frac{1}{n}(A\mathbf{1}_n)\mathbf{1}_n^\top
= \frac{k}{n}\mathbf{1}_n\mathbf{1}_n^\top,\\
TB
&= \frac{1}{n}\mathbf{1}_n(\mathbf{1}_n^\top B)
= \frac{1}{n}\mathbf{1}_n(B\mathbf{1}_n)^\top
= \frac{k}{n}\mathbf{1}_n\mathbf{1}_n^\top.
\end{align*}
Therefore, $AT=TB$ and $J_2(T;A,B)=0$ using the expression of Proposition~\ref{prop:not-equivalent}.

\textbf{Counter example for $\ell_b$:} Let $d$ be non negative and such that $d(x,x)=0$. Recall that $J_b$ is defined as
\begin{equation*}
J_b(T;A,B) = \sum_{i,k=1}^n d\left(A_{ik}, [TBT^\top]_{ik}\right).
\end{equation*}
If we take $A = b \mathbf{1}_n\mathbf{1}_n^\top$ and $T = \frac{1}{n}\mathbf{1}_n\mathbf{1}_n^\top$, we have
\begin{align*}
[TBT^\top]_{ik}
&= \sum_{j,l=1}^n T_{ij} B_{jl} T_{kl}
= \frac{1}{n^2} \sum_{j,l=1}^n B_{jl}
= A_{ik}.
\end{align*}
Therefore, $J_b(T;A,B) = 0$. If $B$ is not a constant matrix, then $A \neq P B P^\top$ for every permutation $P$, which concludes the proof.

\end{proof}

\newpage
\section{Experimental Details}
\label{app:exp}
\subsection{More details on the tasks}
We provide additional detail on the three tasks, and show one representative
(input, prediction, ground truth) triple per task in
Figure~\ref{fig:examples}.

\begin{figure}[h]
    \centering
    \includegraphics[width=\linewidth]{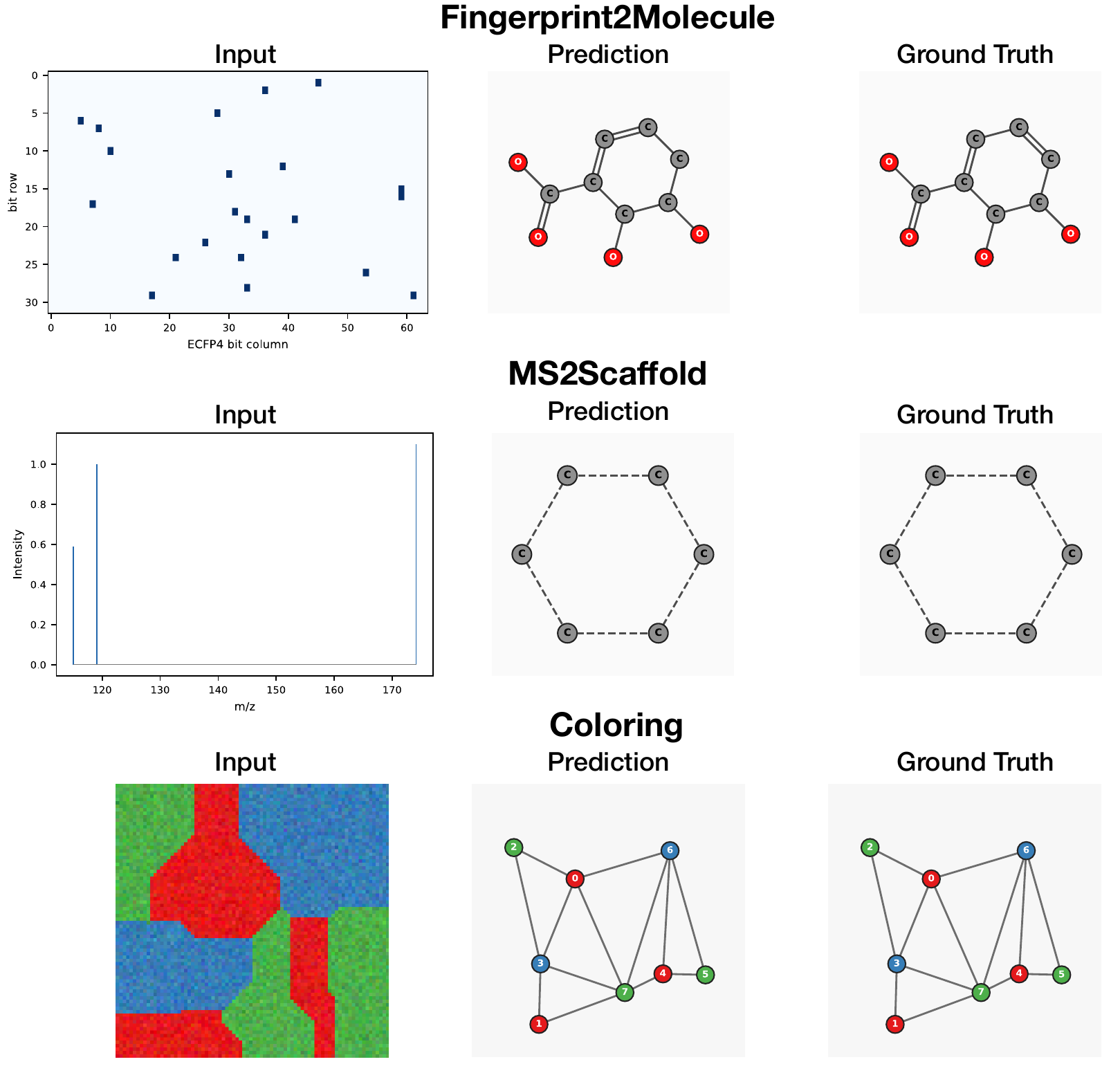}
    \caption{Examples of each supervised graph prediction task.}
    \label{fig:examples}
\end{figure}

\paragraph{Coloring.}
Following \citep{krzakala2024any2graph}, each instance pairs a noisy image of
colored regions with the graph encoding their adjacency: nodes are the colored
regions and edges connect regions that share a border, with node labels giving
the region color. The model must segment the regions and recover their adjacency
structure jointly. 

\paragraph{Fingerprint2Molecule.}
The input is the binary ECFP4 structural fingerprint of a molecule drawn from
PubChem \citep{kim2016pubchem}, and the target is the molecular graph itself
(atoms as labeled nodes, bonds as edges). Since the fingerprint encodes the
presence of local substructures but discards their global arrangement, the model
must reassemble a consistent molecular graph from these overlapping local cues.
The top row of Figure~\ref{fig:examples} shows the fingerprint bit pattern
(displayed as a sparse bit matrix) alongside a predicted graph that closely
matches the ground-truth structure.

\paragraph{MS2Scaffold.}
The input is a tandem (MS/MS) mass spectrum, a sparse set of mass-to-charge
(\(m/z\)) peaks with intensities, and the target is the molecule's
Murcko \emph{scaffold} \cite{landrum2013rdkit}, i.e.\ its core skelleton structure. Predicting the scaffold rather than the full molecule isolates
the core structural-inference problem while remaining a meaningful chemical
target, which is why we propose it as an intermediate step toward full
\emph{de novo} identification. We use the formula-based split of MassSpecGym
\citep{bushuiev2024massspecgym} introduced by \citep{krzakala2026msalign}, which
controls for train-test leakage. The two middle rows of
Figure~\ref{fig:examples} illustrate the range of difficulty: a large
heteroatom-rich macrocycle, where the prediction captures the overall ring but
misses some peripheral atoms, and a simple benzene-like scaffold that is
recovered exactly.

\paragraph{Dataset statistics.}
Table~\ref{tab:dataset_statistics} summarizes the target graph sizes and dataset
cardinalities. For Coloring with maximum capacity $N$, graph sizes are sampled
uniformly from $\{5,\ldots,N\}$ and the dataset contains $20{,}000N$ examples.
The molecular statistics reflect the distributions remaining after
preprocessing and filtering to the configured 32-node capacity.

\begin{table}[h]
\centering
\caption{Dataset statistics. Graph sizes correspond to the number of nodes.}
\small
\begin{tabular}{lrrrr}
\toprule
\textsc{Dataset} & \textsc{Min. Size} & \textsc{Mean Size} & \textsc{Max. Size} & \textsc{\# Samples} \\
\midrule
\textsc{Coloring 10} & 5 & 7.51  & 10 & 200,000 \\
\textsc{Coloring 20} & 5 & 12.51 & 20 & 400,000 \\
\textsc{Coloring 30} & 5 & 17.52 & 30 & 600,000 \\
\textsc{Coloring 40} & 5 & 22.52 & 40 & 800,000 \\
\textsc{Coloring 50} & 5 & 27.50 & 50 & 1,000,000 \\
\textsc{Coloring 60} & 5 & 32.49 & 60 & 1,200,000 \\
\midrule
\textsc{MS2Scaffold} & 3 & 17.25 & 32 & 168,573\\
\textsc{Fingerprint2Molecule}   & 1 & 21.82 & 32 & 80,897,368 \\
\bottomrule
\end{tabular}
\label{tab:dataset_statistics}
\end{table}

\subsection{Implementation Details}

\paragraph{Architecture.}
All tasks use a modality-specific encoder followed by a Transformer decoder
with one learned query per output node. For \textbf{Fingerprint2Molecule}, we
tokenize the active entries of a radius-2, 2048-bit ECFP4 fingerprint and
prepend a learned start token. For \textbf{MS2Scaffold}, we encode at most 128
annotated peaks without positional encodings and condition the encoder on the
collision energy. Furthermore, following Any2Graph~\cite{krzakala2024any2graph}, for both
molecular tasks we apply one-hop feature diffusion, augmenting the node-label
matrix $F$ with $AF$, where $A$ is the adjacency matrix. The model predicts this
diffused component through an auxiliary node head, and uses it in both matching
and reconstruction. This proved to increase performance in these hard tasks.
For \textbf{Coloring}, we use SegFormer-B1
features~\cite{xie2021segformer} and apply random horizontal reflections and
quarter-turn rotations during training. Table~\ref{tab:core_hyperparameters}
summarizes the remaining architecture and training settings.

\paragraph{Amortized alignment.}
Predicted and target node states are independently projected and compared using
pairwise $\ell_1$ distances. Each cost matrix is normalized by its sum before applying log-domain Sinkhorn. For labeled graphs, the objective in Eq.\eqref{eqloss}
combines presence, node-label, edge-label, adjacency, and marginal penalties
with weights $(1,1,0.2,0.5,1)$. The Sinkhorn regularization $\varepsilon$ is
selected separately for each task. The mirror baseline uses the same predictor
and reconstruction objective but computes a detached alignment for every
example.

\begin{table}[t]
\centering
\caption{Core architecture and training hyperparameters.}
\label{tab:core_hyperparameters}
\small
\setlength{\tabcolsep}{4pt}
\begin{tabular}{
    @{}
    p{0.27\linewidth}
    p{0.32\linewidth}
    p{0.32\linewidth}
    @{}
}
\toprule
\textbf{Parameter} &
\textbf{Molecular tasks} &
\textbf{Coloring} \\
\midrule
Input encoder
& 3-layer Transformer
& SegFormer-B1 \\

Encoder dimension
& 512
& 128 \\

Node decoder
& 3 layers, width 512
& 5 layers, width 256 \\

Attention heads
& 8
& 4 \\

Target GNN
& 3 layers, width 512
& 5 layers, width 128 \\

Laplacian PE dimension
& 8
& 8 \\

Matcher dimension
& 128
& 256 \\

Dropout
& 0.1
& 0.1 \\

Sinkhorn iterations
& 20
& 20 \\

Optimizer
& AdamW
& AdamW \\

Learning rate
& $10^{-4}$
& $10^{-4}$ \\

Minimum learning rate
& $10^{-5}$ 
& $10^{-5}$ \\

Learning-rate schedule
& 5\% warm-up + cosine
& 5\% warm-up + cosine \\

Batch size
& 128
& Size-dependent \\

Maximum epochs
&  100 ms2scaffold / 10 fp2graph
& Size-dependent \\

Gradient clipping
& 1.0
& 0.1 \\

Precision
& FP32
& FP32 \\

Hardware
& 1 NVIDIA V100
& 1 NVIDIA V100 \\
\bottomrule
\end{tabular}
\end{table}

\begin{table}[t]
\centering
\caption{Size-dependent Coloring hyperparameters.}
\label{tab:coloring_training_budgets}
\small
\setlength{\tabcolsep}{0pt}
\begin{tabular*}{\linewidth}{
    @{\extracolsep{\fill}}
    l
    rrrrrr
    @{}
}
\toprule
Maximum graph size $n$
& 10 & 20 & 30 & 40 & 50 & 60 \\
\midrule
Batch size
& 1536 & 1536 & 512 & 256 & 128 & 128 \\

Epochs
& 120 & 120 & 100 & 50 & 30 & 20 \\
\bottomrule
\end{tabular*}
\end{table}

\subsection{More details on the results}

\paragraph{Model selection.}
We performed small task-specific validation sweeps over the input encoder,
decoder, target encoder, positional features, and matcher representation. We
then fixed the best validation configuration for each task, as summarized in
Table~\ref{tab:core_hyperparameters}.

\paragraph{Main comparison.}
For Table~\ref{tab:main_results}, we select $\varepsilon$ on validation data.
The selected values are $3\times10^{-5}$ for both molecular tasks,
$4.5\times10^{-5}$ for Coloring 10, and $7.5\times10^{-6}$ for Coloring 20.
All models otherwise follow Tables~\ref{tab:core_hyperparameters}
and~\ref{tab:coloring_training_budgets}. Due to its substantially higher
per-example matching cost, Any2Graph is trained for half the corresponding
epoch budget, approximately the largest budget compatible with our 20-hour
limit.

\paragraph{Efficiency--quality frontier.}
Figure~\ref{fig:main} (left) compares the following configurations on
Coloring 20.

\begin{table}[H]
\centering
\caption{Configurations used for the efficiency--quality comparison.}
\label{tab:efficiency_grid}
\scriptsize
\setlength{\tabcolsep}{3pt}
\begin{tabular}{@{}lcccc@{}}
\toprule
\textbf{Method} & \textbf{Regularization}
& \textbf{$K_{\mathrm{in}}$}
& \textbf{$K_{\mathrm{out}}$}
& \textbf{Marginal KL} \\
\midrule
Mirror
& $\tau=0.1$
& $\{1,10,100\}$
& $\{10,100\}$
& N/A \\
Matcher
& $\varepsilon=10^{-5}$
& $\{5,10,50,100\}$
& -- & No \\
Matcher + KL
& $\varepsilon=10^{-5}$
& $\{5,10,50,100\}$
& -- & Yes \\
\bottomrule
\end{tabular}
\end{table}

We plot the frontier of edit-like distance against seconds per sample for each
method. The mirror configuration
$K_{\mathrm{in}}=K_{\mathrm{out}}=100$ is omitted because its runtime is an
extreme outlier.

\paragraph{Scaling with graph capacity.}
Each Coloring dataset samples graph sizes
$m\sim\mathrm{Uniform}\{5,\ldots,N\}$, where $N$ is the maximum graph size.
Although individual graphs may be smaller, the matcher operates on $N$
predicted and $N$ padded target slots and normalizes the resulting
$N\times N$ cost matrix by its sum. If the unnormalized cost distribution were
stable across capacities, this would shrink the relevant cost contrasts
approximately as $N^{-2}$. We therefore use
\[
    \bar{\varepsilon}_N
    =
    \varepsilon_{20}\left(\frac{20}{N}\right)^2,
    \qquad
    \varepsilon_{20}=7.5\times10^{-6},
\]
only to center a local validation sweep.

\begin{table}[H]
\centering
\caption{Coloring epsilon candidates, in units of $10^{-6}$.}
\label{tab:coloring_epsilon_sweep}
\scriptsize
\setlength{\tabcolsep}{3pt}
\begin{tabular*}{\columnwidth}{
    @{\extracolsep{\fill}}lrrrrrr@{}
}
\toprule
Maximum size $N$ & 10 & 20 & 30 & 40 & 50 & 60 \\
\midrule
$0.75\bar{\varepsilon}_N$
& 22.5 & 5.625 & 2.50 & 1.406 & 0.90 & 0.625 \\
$\bar{\varepsilon}_N$
& 30.0 & 7.500 & 3.33 & 1.875 & 1.20 & 0.833 \\
$1.5\bar{\varepsilon}_N$
& 45.0 & 11.25 & 5.00 & 2.813 & 1.80 & 1.250 \\
\bottomrule
\end{tabular*}
\end{table}

The scaling is approximate because the distribution of graph sizes, padding
costs, and learned cost contrasts also changes with $N$. We therefore select
the best candidate using validation data. The mirror solver does not use
$\varepsilon$ and is evaluated with $\tau=0.1$ and 20 inner and outer
iterations.


\end{document}